\documentclass{article}

\usepackage[final]{nips_2016}

\usepackage[utf8]{inputenc} 
\usepackage[T1]{fontenc}    
\usepackage{hyperref}       
\usepackage{url}            
\usepackage{booktabs}       
\usepackage{amsfonts}       
\usepackage{nicefrac}       
\usepackage{microtype}      

\usepackage{graphicx} 
\usepackage{amsmath}
\usepackage{amssymb}
\usepackage[algoruled]{algorithm2e}
\usepackage{caption}
\usepackage{subcaption}
\usepackage{color}
\usepackage{amsthm}

\title{Solving Marginal MAP Problems with NP Oracles and Parity
  Constraints}

\author{
Yexiang Xue\\
Department of Computer Science\\
Cornell University\\
\texttt{yexiang@cs.cornell.edu}\\
\And
Zhiyuan Li\thanks{This research was done when Zhiyuan Li was an exchange
  student at Cornell University.}\\
Institute of Interdisciplinary Information Sciences\\
Tsinghua University\\
\texttt{lizhiyuan13@mails.tsinghua.edu.cn}\\
\And
Stefano Ermon\\
Department of Computer Science\\
Stanford University\\
\texttt{ermon@cs.stanford.edu}\\ 
\And
Carla P. Gomes, Bart Selman\\
Department of Computer Science\\
Cornell University\\
\texttt{\{gomes,selman\}@cs.cornell.edu}\\
}

\DeclareMathOperator*{\argmin}{argmin}
\begin{document}

\def\gm{{\mathcal{M}}}
\def\gmp{{\mathcal{M}^\prime}}
\def\gme{{\mathcal{E}}}
\def\gmep{{\mathcal{E}^\prime}}

\def\bn{{\mathcal{N}}}
\def\bnp{{\mathcal{N}^\prime}}

\def\n(#1){\bar{#1}}
\def\u{{\bf u}}
\def\a{{\bf a}}
\def\U{{\bf U}}
\def\v{{\bf v}}
\def\V{{\bf V}}
\def\x{{\bf x}}
\def\X{{\bf X}}
\def\y{{\bf y}}
\def\Y{{\bf Y}}
\def\z{{\bf z}}
\def\Z{{\bf Z}}

\newcommand{\mc}[2]{\mathbb{C}^{#1\times #2}}
\newcommand{\mr}[2]{\mathbb{R}^{#1\times #2}}
\newcommand{\mrt}[3]{\mathbb{R}^{#1\times #2\times #3}}
\newcommand{\mrN}{\mr{I_1}{I_2\times \dots I_N}}
\newcommand{\rank}{\mathrm{rank}}
\newcommand{\diag}{\mathrm{diag}}
\newcommand{\nf}[2]{\nicefrac{#1}{#2}}
\newcommand{\vct}[1]{\bm{#1}}
\newcommand{\mat}[1]{\bm{#1}}
\newcommand{\frob}[1]{\|#1\|_2}
\newcommand{\F}[1]{\|#1\|^2_F}
\newcommand{\p}[2]{\textrm{Pr}[ #1\in #2] }
\newcommand{\pe}[2]{\mathop{\textrm{Pr}}\limits_{#1}\left[#2\right] }
\newcommand{\pr}[1]{\textrm{Pr} \left[#1\right] }
\newcommand{\E}[1]{\mathbb{E}\left[#1\right]}

\newcommand{\setW}{{\mathcal{W}}}
\newcommand{\setX}{{\mathcal{X}}}
\newcommand{\setA}{{\mathcal{A}}}
\newcommand{\proofsketch}{\textit{Proofsketch. }}

\newcommand{\one}[1]{\mathbf{1}_{#1}}
\newcommand{\algoref}[1]{\textbf{Algorithm \ref{#1}}}

\newcommand{\stefano}[1]{{\color{red} [SE: {#1}]}}

\newtheorem{thm}{Theorem}[section]
\newtheorem{lem}[thm]{Lemma}
\newtheorem{crl}[thm]{Corollary}
\newtheorem{clm}[thm]{Claim}
\newtheorem{exmp}[thm]{Example}
\newtheorem{defi}[thm]{Definition}
\newtheorem{prop}[thm]{Proposition}

\maketitle

\begin{abstract}
Arising from many applications at the intersection of decision-making
and machine learning, Marginal Maximum A Posteriori (Marginal MAP)
problems unify the two main classes of inference, namely
\emph{maximization} (optimization) and \emph{marginal inference}
(counting), and are believed to have higher complexity than both of
them.
We propose $\mathtt{XOR\_MMAP}$, a novel approach to solve the
Marginal MAP problem, which represents the intractable counting
subproblem with queries to NP oracles, subject to additional parity
constraints.
$\mathtt{XOR\_MMAP}$ provides a constant factor approximation to the
Marginal MAP problem, by encoding it as a single optimization in a 
polynomial size of the original problem.
We evaluate our approach in several machine learning and
decision-making applications, and show that our approach outperforms
several state-of-the-art Marginal MAP solvers.

\end{abstract}

\section{Introduction}


Typical inference queries to make predictions and learn probabilistic 
models from data include the \emph{maximum a posteriori} (MAP)
inference task, which computes the most likely assignment of a set of
variables, as well as the \emph{marginal inference} task, which
computes the probability of an event according to the model.  Another
common query is the Marginal MAP (MMAP) problem, which involves both
\emph{maximization} (optimization over a set of variables) and
\emph{marginal inference} (averaging over another set of
variables). 

Marginal MAP problems arise naturally in many machine
learning applications. For example, learning latent variable models
can be formulated as a MMAP inference problem, where the goal is to
optimize over the model's parameters while marginalizing all the
hidden variables.
MMAP problems also arise naturally in the context of decision-making under uncertainty, where the goal is to find a decision
(optimization) that performs well on average across multiple
probabilistic scenarios (averaging).
%

%

The Marginal MAP problem is known to be
$\mbox{NP}^{\mbox{PP}}$-complete \cite{Park04MMAPComplexity}, which is
commonly believed to be harder than both MAP inference (NP-hard)
and marginal inference (\#P-complete).
As supporting evidence, MMAP problems are NP-hard even on tree
structured probabilistic graphical models \cite{LiuI13VariationalMMAP}.
Aside from attempts to solve MMAP problems exactly
\cite{Park2002MMAPExact,Radu14AndOrMMAP,Marinescu2015AOBB,Maua2012AnytimeMAP},
previous approximate approaches fall into two categories, in
general. The core idea of approaches in both categories is to
effectively approximate the intractable marginalization, which often involves averaging over an exponentially large number of scenarios.
One class of approaches
\cite{LiuI13VariationalMMAP,Jiang11MessagePassingMMAP,Wei15DecompositionBoundMMAP,Lee2016AOBB}
use variational forms to represent the intractable sum. Then
the entire problem can be solved with message passing algorithms,
which correspond to searching for the best variational approximation
in an iterative manner. 
As another family of approaches, Sample Average Approximation (SAA)
\cite{Sheldon10NetworkDesign,Xue15ScheduleCascade} uses a fixed set of
samples to represent the intractable sum, which then transforms the
entire problem into a restricted optimization, only considering a
finite number of samples.
Both approaches treat the optimization and marginalizing components
separately. However, we will show that by solving these two tasks in
an integrated manner, we can obtain significant computational
benefits.

Ermon et al. \cite{Ermon13Wish,Ermon14LowDensityParity} recently proposed an alternative approach to approximate intractable counting
problems. Their key idea is a mechanism to transform
a counting problem into a series of optimization problems,
each corresponding to the original problem subject to
randomly generated XOR constraints.
Based on this mechanism, they developed an algorithm providing a
constant-factor approximation to the counting (marginalization) problem.

We propose a novel algorithm, called $\mathtt{XOR\_MMAP}$, which
approximates the intractable sum with a series of optimization problems, which in turn are folded into the global optimization task. Therefore, we effectively reduce the original MMAP inference to \emph{a single joint  optimization} of polynomial size of the original problem.


We show that $\mathtt{XOR\_MMAP}$ provides a constant factor approximation to the Marginal MAP problem. Our approach also provides upper and lower bounds on the final result. The quality of the bounds can be improved incrementally with increased computational effort. 

%
We evaluate our algorithm on unweighted SAT instances and on 
weighted Markov Random Field models, comparing our algorithm with
variational methods, as well as sample average approximation. We also show the
effectiveness of our algorithm on applications in computer vision with
deep neural networks and in computational sustainability.
Our sustainability application shows how MMAP problems are also found
in scenarios of searching for optimal policy interventions to maximize
the outcomes of probabilistic models. 
%
%
As a first example, we consider a network design application to
maximize the spread of cascades \cite{Sheldon10NetworkDesign}, which
include modeling animal movements or information diffusion in social
networks.
In this setting, the marginals of a probabilistic decision model
represent the probabilities for a cascade to reach certain target
states (averaging), and the overall network design problem is to make
optimal policy interventions on the network structure to maximize the
spread of the cascade (optimization).
As a second example, in a crowdsourcing domain, probabilistic models are
used to model people's behavior. The organizer would like to find an
optimal incentive mechanism (optimization) to steer people's
effort towards crucial tasks, taking into account the probabilistic
behavioral model  (averaging) \cite{xue2015IncentivizingTR}.

We show that $\mathtt{XOR\_MMAP}$ is able to find considerably better
solutions than those found by previous methods, as well as provide
tighter bounds.

\section{Preliminaries}
%
{\bf Problem Definition} Let $\setA =\{0,1\}^m$ be the set of 
all possible assignments to binary variables $a_1,\ldots,a_m$ and $\setX =\{0,1\}^n$ be the set of assignments to  binary variables $x_1,\ldots,x_n$.
Let $w(x,
a):\setX\times \setA \rightarrow \mathbb{R}^+$ be a function that maps
every assignment to a non-negative value. Typical queries over a
probabilistic model include the \emph{maximization} task, which
requires the computation of $\max_{a \in \setA} w(a)$, and the \emph{marginal
  inference} task $\sum_{x \in \setX} w(x)$, which sums over $\setX$.

Arising naturally from
many machine learning applications, the following \emph{Marginal Maximum A Posteriori} (Marginal MAP) problem is a joint inference task, which
combines the two aforementioned inference tasks:
\begin{equation}
  \max_{a \in \setA}~ \sum_{x \in \setX} w(x, a).
  \label{eq:general}
\end{equation}
We consider the case where the counting problem $\sum_{x \in \setX}
w(x, a)$ and the maximization problem $\max_{a \in \setA} \#
w(a)$ are defined over sets of exponential size, therefore both are intractable in general.
{\bf Counting by Hashing and Optimization} Our approach is based on a
recent theoretical result that transforms a counting problem to a
series of optimization
problems~\cite{Ermon13Wish,Ermon14LowDensityParity,belle2015hashing,achlioptasstochastic}.
A family of functions $\mathcal{H} =\{h:\{0,1\}^n \rightarrow
\{0,1\}^k\}$ is said to be \textit{pairwise independent} if the
following two conditions hold for any function $h$ randomly chosen
from the family $\mathcal{H}$: (1) $\forall x \in \{0, 1\}^n$, the
random variable $h(x)$ is uniformly distributed in $\{0, 1\}^k$ and
(2) $\forall x_1, x_2 \in \{0, 1\}^n$, $x_1 \neq x_2$, the random
variables $h(x_1)$ and $h(x_2)$ are independent.

We sample matrices $A \in \{0,1\}^{k\times n}$ and vector $b \in
\{0,1\}^k$ uniformly at random to form the function family
$\mathcal{H}_{A,b}=\{h_{A,b}: h_{A,b}(x) = Ax + b \mbox{ mod } 2\}$. It is possible to show that $\mathcal{H}_{A,b}$ is pairwise independent~\cite{Ermon13Wish,Ermon14LowDensityParity}.
Notice that in this case, each function $h_{A,b}(x) = Ax + b
\mbox{ mod } 2$ corresponds to $k$ parity constraints.
One useful way to think about pairwise independent functions is to
imagine them as functions that randomly project elements in
$\{0,1\}^n$ into $2^k$ buckets.
%
Define $B_h(g) = \{x\in\{0,1\}^n: h_{A,b}(x) = g\}$ to be
a ``bucket'' that includes all elements in $\{0,1\}^n$ whose mapped
value $h_{A,b}(x)$ is vector $g$ ($g \in \{0,1\}^k$). 
Intuitively, if we randomly sample a function $h_{A,b}$ from a pairwise
independent family, then we get the following: $x\in\{0,1\}^n$ has an
equal probability to be in any bucket $B(g)$, and the bucket locations
of any two different elements $x, y$ are independent.


\section{XOR\_MMAP Algorithm}
\begin{algorithm}[t]
Sample function $h_k:\mathcal{X}\to \{0,1\}^{k}$ from a pair-wise independent function family\;
Query an NP Oracle on whether\\
$~~~~~~~~~~\setW(a_0, h_k) = \{x \in \setX: w(a_0, x) = 1, h_k(x) = {\bf{0}}\}$ is empty\;
Return \textbf{true} if $\setW(a_0, h_k) \neq \emptyset$, otherwise return \textbf{false}.
\caption{$\mathtt{XOR\_Binary}$($w: \setA \times \mathcal{X}\to \{0,1\}$, $a_0$, $k$)}\label{alg:naive_COMP}
\end{algorithm}
\setlength{\textfloatsep}{3pt}

\vspace{-0.05in}
\subsection{Binary Case}
\vspace{-0.05in}

We first solve the Marginal MAP problem for the binary case, in which the function $w:\mathcal{A}\times \mathcal{X}\to \{0,1\}$ outputs either 0 or 1. We will extend the result to the weighted case in the next section.
Since $a\in\setA$ often represent decision variables when MMAP problems are used in decision making, we call a fixed assignment to vector $a=a_0$ a ``solution strategy''. To simplify the notation, we use $\setW(a_0)$ to represent the set $\{x \in \setX: w(a_0, x) = 1\}$, and use $\setW(a_0, h_k)$ to represent the set $\{x \in \setX: w(a_0, x) = 1\mbox{ and } h_k(x) = {\bf{0}}\}$, in which $h_k$ is sampled from a pairwise independent function family that maps $\setX$ to $\{0,1\}^k$. We write $\#w(a_0)$ as shorthand for the count $|\{x \in \setX: w(a_0, x) = 1\}| = \sum_{x\in \mathcal{X}} w(a_0,x)$.
Our algorithm depends on the following result:
\begin{thm}\label{thm:naive_COMP}
  (Ermon et. al.\cite{Ermon13Wish}) For a fixed solution strategy $a_0\in \mathcal{A}$,
  \begin{itemize}
    \item Suppose $\#w(a_0)\geq 2^{k_0}$, then for any $k\leq k_0$, with probability $1-\frac{2^c}{(2^c-1)^2}$, Algorithm $\mathtt{XOR\_Binary}$($w, a_0, k-c$)=\textbf{true}.
    \item Suppose $\#w(a_0) < 2^{k_0}$, then for any $k\geq k_0$, with probability $1-\frac{2^c}{(2^c-1)^2}$, Algorithm $\mathtt{XOR\_Binary}$($w, a_0, k+c$)=\textbf{false}.
  \end{itemize}
\end{thm}
To understand Theorem~\ref{thm:naive_COMP}  intuitively, we can
think of $h_k$ as a function that maps every element in set $\setW(a_0)$
into $2^k$ buckets. Because $h_k$ comes from a pairwise independent
function family, each element in $\setW(a_0)$ will have an equal
probability to be in any one of the $2^k$ buckets, and the buckets in
which any two elements end up are mutually independent.
Suppose the count of solutions for a fixed strategy $\#w(a_0)$ is
$2^{k_0}$, then with high probability, there will be at least one
element located in a randomly selected bucket if the number of buckets
$2^k$ is less than $2^{k_0}$. Otherwise, with high probability there
will be no element in a randomly selected bucket.
%

Theorem~\ref{thm:naive_COMP} provides us with a way to obtain a
rough count on $\#w(a_0)$ via a series of tests on whether
$\setW(a_0,h_k)$ is empty, subject to extra parity functions
$h_k$. This transforms a counting problem to a series of NP
queries, which can also be thought of as optimization queries. This transformation is extremely helpful for the Marginal MAP problem. As noted earlier, the main challenge for the
marginal MAP problem is the intractable sum embedded in the
maximization. Nevertheless, the whole problem can be re-written as a
single optimization if the intractable sum can be approximated well by solving an optimization problem over the same domain.

We therefore design Algorithm $\mathtt{XOR\_MMAP}$, which is able to
provide a constant factor approximation to the Marginal MAP problem.  The whole algorithm is shown in Algorithm~\ref{alg:MCPH}.
In its main procedure $\mathtt{XOR\_K}$, the algorithm transforms the Marginal
MAP problem into an optimization over the sum of $T$ replicates of the original function $w$.
%
%
Here, $x^{(i)}\in \setX$ is a replicate of the original $x$, and $w(a,
x^{(i)})$ is the original function $w$ but takes $x^{(i)}$ as one of
the inputs.
All replicates share common input $a$.
In addition, each replicate is subject to an independent set of parity
constraints on $x^{(i)}$.
%
Theorem~\ref{thm:constant_approx} states
that $\mathtt{XOR\_MMAP}$ provides a constant-factor approximation to
the Marginal MAP problem:
\begin{thm}
  For $T\geq \frac{m\ln2 +\ln(n/\delta)}{\alpha^*(c)}$, with probability $1-\delta$, $\mathtt{XOR\_MMAP}$($w,\log_2|\mathcal{X}|,\log_2|\mathcal{A}|,T$) outputs a $2^{2c}$-approximation to the Marginal MAP problem: $\max_{a\in \mathcal{A}}\#w(a)$. $\alpha^*(c)$ is a constant. \label{thm:constant_approx} 
\end{thm}
Let us first understand the theorem in an
intuitive way. Without losing generality, suppose the optimal value $\max_{a\in \mathcal{A}}\#w(a) =2^{k_0}$. Denote $a^*$ as the optimal solution, ie, $\#w(a^*)=2^{k_0}$.
According to Theorem~\ref{thm:naive_COMP}, the set $\setW(a^*,h_k)$
has a high probability to be non-empty, for any function $h_k$ that
contains $k < k_0$ parity constraints. In this case, the optimization
problem $\max_{x^{(i)}\in \mathcal{X}, h^{(i)}_k(x^{(i)})=\mathbf{0}}
w(a^{*},x^{(i)})$ for one replicate $x^{(i)}$ almost always returns
1. Because $h^{(i)}_k$ ($i=1\ldots T$) are sampled independently, the
sum $\sum_{i=1}^{T}w(a^{*},x^{(i)})$ is likely to be larger than
$\lceil T/2\rceil$, since each term in the sum is likely to be 1 (under 
the fixed $a^{*}$).
Furthermore, since $\mathtt{XOR\_K}$ maximizes this sum over all possible
strategies $a\in\setA$, the sum it finds will be at least as good as
the one attained at $a^{*}$, which is already over $\lceil
T/2\rceil$. Therefore, we conclude that when $k < k_0$,
$\mathtt{XOR\_K}$ will return $\mathbf{true}$ with high probability.

\begin{table}[t]
\begin{minipage}[t]{.5\textwidth}
\begin{algorithm}[H]
Sample $T$ pair-wise independent hash functions \\
~~~~~~~~~~$h_k^{(1)},h_k^{(2)},\ldots,h_k^{(T)}:\mathcal{X}\to \{0,1\}^{k}$\;
Query Oracle
\begin{equation}
\begin{split}
   \max\limits_{a\in\mathcal{A},x^{(i)}\in \mathcal{X}}& \sum_{i=1}^{T}w(a,x^{(i)})\\
  \mbox{s.t.}& \quad h^{(i)}_k(x^{(i)})=\mathbf{0}, i=1,\ldots,T.
\end{split}
\label{eq:xork}
\end{equation}
Return \textbf{true} if the max value is larger than $\lceil T/2\rceil$, otherwise return \textbf{false}.
\caption{$\mathtt{XOR\_K}$($w:\mathcal{A}\times \mathcal{X}\to \{0,1\},k,T$)}\label{alg:COMP}
\end{algorithm}
\end{minipage}
~~~~~
\begin{minipage}[t]{.45\textwidth}
\begin{algorithm}[H]
$k = n$\;
\While{$k>0$}{
  \If{$\mathtt{XOR\_K}$($w,k,T$)}{
    Return $2^k$\;
  }
  $k\gets k-1$\;
}
Return 1\;
\caption{$\mathtt{XOR\_MMAP}$($w:\mathcal{A}\times \mathcal{X}\to \{0,1\}$,$n=\log_2|\mathcal{X}|$,$m=\log_2|\mathcal{A}|$,$T$)}\label{alg:MCPH}
\end{algorithm}
\end{minipage}
\end{table}

%
We can develop similar arguments to conclude that
$\mathtt{XOR\_K}$ will return $\mathbf{false}$ with high probability
when more than $k_0$ XOR constraints are added.
%
Notice that replications and an additional union bound argument are
necessary to establish the probabilistic guarantee in this case. 
As a counter-example, suppose function $w(x,a)=1$ if and only if $x=a$, otherwise $w(x,a)=0$ ($m=n$ in this case). If we set the number of replicates $T=1$, then $\mathtt{XOR\_K}$ will almost always return 1 when $k<n$, which suggests that there are $2^n$ solutions to the MMAP problem. Nevertheless, in this case the true optimal value of $\max_x \#w(x,a)$ is 1, which is far away from $2^n$. This suggests that at least two replicates are needed.

\begin{lem}\label{lem:COMP}
  For $T \geq \frac{\ln2\cdot m+\ln(n/\delta)}{\alpha^*(c)}$ , procedure $\mathtt{XOR\_K}$($w$,$k$,$T$) satisfies:
  \begin{itemize}
    \item Suppose $\exists a^* \in \mathcal{A}$, s.t. $\#w(a^*)\geq 2^{k}$, then with probability $1-\frac{\delta}{n2^m}$, $\mathtt{XOR\_K}(w,k-c,T)$ returns $\mathbf{true}$.
    \item Suppose $\forall a_0\in \mathcal{A}$, s.t. $\#w(a_0) < 2^{k}$, then with probability  $1-\frac{\delta}{n}$, $\mathtt{XOR\_K}(w,k+c,T)$ returns $\mathbf{false}$.
  \end{itemize}
\end{lem}
\begin{proof}
  \textbf{Claim 1:} If there exists such $a^*$ satisfying
  $\#w(a^*)\geq 2^{k}$, pick $a_0=a^*$. Let $X^{(i)}(a_0) =
  \max_{x^{(i)}\in \mathcal{X}, h^{(i)}_{k-c}(x^{(i)})=\mathbf{0}}
  w(a_0,x^{(i)})$, for $i=1\ldots, T$.  From
  Theorem~\ref{thm:naive_COMP}, $X^{(i)}(a_0)=1$ holds with
  probability $1-\frac{2^c}{(2^c-1)^2}$. Let
  $\alpha^*(c)=D(\frac{1}{2}\|\frac{2^c}{(2^c-1)^2})$. By Chernoff
  bound, we have
  {\small
    \begin{equation}
    \pr{\max\limits_{a\in\mathcal{A}}\sum_{i=1}^{T}X^{(i)}(a) \leq T/2}\leq \pr{\sum_{i=1}^{T}X^{(i)}(a_0) \leq T/2}\leq e^{-D(\frac{1}{2}\|\frac{2^c}{(2^c-1)^2})T}=e^{-\alpha^*(c)T},
    \end{equation}
   }
  where
  {\small
    \begin{equation*}
      D\left(\frac{1}{2}\|\frac{2^c}{(2^c-1)^2}\right)= 2\ln(2^c-1)-\ln2-\frac{1}{2}\ln(2^c)-\frac{1}{2}\ln((2^c-1)^2-2^c) \geq (\frac{c}{2}-2)\ln2.
    \end{equation*} 
  }
For $T \geq \frac{\ln2\cdot m+\ln(n/\delta)}{\alpha^*(c)}$, we have $e^{-\alpha^*(c)T}\leq \frac{\delta}{n2^m}$. Thus, with probability $1-\frac{\delta}{n2^m}$, we have $\max\limits_{a\in\mathcal{A}} \sum_{i=1}^{T}X^{(i)}(a) > T/2$, which implies that $\mathtt{XOR\_K}(w,k-c,T)$ returns $\mathbf{true}$.

\textbf{Claim 2:} The proof is almost the same as Claim 1, except that we need to use a union bound to let the property hold for all $a\in\mathcal{A}$ simultaneously. As a result, the success probability will be $1-\frac{\delta}{n}$ instead of $1-\frac{\delta}{n2^m}$. The proof is left to supplementary materials. 
\end{proof}

\begin{proof}
(Theorem~\ref{thm:constant_approx}) With probability $1-n\frac{\delta}{n}=1-\delta$, the output of $n$ calls of $\mathtt{XOR\_K}$($w,k,T$) (with different $k=1\ldots n$) all satisfy the two claims in Lemma~\ref{lem:COMP} simultaneously. Suppose $\max\limits_{a\in \mathcal{A}}\#w(a)\in [2^{k_0},2^{k_0+1})$, we have (i) $\forall k\geq k_0+c+1$,  $\mathtt{XOR\_K}(w,k,T)$ returns $\mathbf{false}$,
(ii) $\forall k\leq k_0-c$, $\mathtt{XOR\_K}(w,k,T)$ returns $\mathbf{true}$.
  Therefore, with probability $1-\delta$, the output of $\mathtt{XOR\_MMAP}$ is guaranteed to be among $2^{k_0-c}$ and $2^{k_0+c}$.
\end{proof}




The approximation bound in Theorem~\ref{thm:constant_approx} is a
worst-case guarantee. We can obtain a tight bound 
(e.g. 16-approx) with a large number of $T$ replicates.
Nevertheless, we keep a small $T$, therefore a loose bound, in our
experiments, after trading between the formal guarantee and the empirical
complexity. In practice, our method performs well, even with loose bounds.
Moreover, $\mathtt{XOR\_K}$ procedures with different input $k$ are
not uniformly hard. We therefore can run them in parallel. We can obtain a looser bound at any given time, based on all completed $\mathtt{XOR\_K}$ procedures.
Finally, if we have access to a polynomial approximation algorithm for
the optimization problem in $\mathtt{XOR\_K}$, we can propagate this
bound through the analysis, and again get a guaranteed bound, albeit
looser for the MMAP problem.

{\bf Reduce the Number of Replicates} 
We further develop a few variants of $\mathtt{XOR\_MMAP}$ in the
supplementary materials to reduce the number of replicates, as
well as the number of calls to the $\mathtt{XOR\_K}$ procedure, while
preserving the same approximation bound.

{\bf Implementation} We solve the optimization problem in
$\mathtt{XOR\_K}$ using Mixed Integer Programming (MIP). Without
losing generality, we assume $w(a,x)$ is an indicator variable, which
is 1 iff $(a,x)$ satisfies constraints represented in
Conjunctive Normal Form (CNF). We introduce extra variables to
represent the sum $\sum_i w(a, x^{(i)})$ which is left in
the supplementary materials.
The XORs in Equation~\ref{eq:xork} are encoded as MIP 
constraints using the Yannakakis encoding, similar as in
\cite{Ermon2013CplexWish}.




\subsection{Extension to the Weighted Case}

In this section, we study the more general case, where $w(a,x)$ takes non-negative real numbers instead of integers in $\{0,1\}$.
%
Unlike in \cite{Ermon13Wish}, we choose to build our proof from the unweighted case because it can effectively avoid modeling the median of an array of numbers~\cite{ermon2013embed}, which is difficult to encode in integer programming.
We noticed recent work \cite{Chakraborty2015WeightedCounting}. It is
related but different from our approach.
%
Let $w:\mathcal{A}\times \mathcal{X}\to \mathbb{R}^+$, and $M=\max_{a,x}w(a,x)$.
\begin{defi}
  We define the embedding $\mathcal{S}_a(w,l)$ of $\mathcal{X}$ in $\mathcal{X}\times\{0,1\}^{l}$ as:
\begin{equation}
  \mathcal{S}_a(w,l)=\left\{(x,y)| \forall 1\leq i \leq l, \frac{w(a,x)}{M}\leq \frac{2^{i-1}}{2^l}\Rightarrow y_i=0\right\}.
\end{equation}
\label{def:sandwich}
\end{defi}

\begin{lem}
  Let $w'_l(a,x,y)$ be an indicator variable which is 1 if and only if $(x,y)$ is in $\mathcal{S}_a(w,l)$, i.e.,
    $w'_l(a,x,y) = \one{(x,y) \in \mathcal{S}_a(w,l)}$. 
We claim that
\begin{equation}
  \max_a\sum_xw(a,x)\leq \frac{M}{2^l}\max_a \sum_{(x,y)}w'_l(a,x,y) \leq 2\max_a\sum_xw(a,x)+M2^{n-l}.\footnote{  If $w$ satisfy the property that $\min_{a,x}w(a,x)\geq 2^{-l-1}M$, we don't have the $M2^{n-l}$ term. }
\end{equation}\label{lem:weighted}
\end{lem}
\begin{proof}
 Define $S_a(w,l,x_0)$ as the set of $(x,y)$ pairs within the set $S_a(w,l)$ and $x=x_0$, ie, $S_a(w,l,x_0) = \{(x,y) \in S_a(w,l): x=x_0\}$. It is not hard to see that $\sum_{(x,y)}w'_l(a,x,y) = \sum_{x} |S_a(w,l,x)|$. In the following, first we are going to establish the relationship between $|S_a(w,l,x)|$ and $w(a,x)$. Then we use the result to show the relationship between $\sum_{x} |S_a(w,l,x)|$ and $\sum_x w(x,a)$. 
Case (i): If $w(a,x)$ is sandwiched between two exponential levels: $\frac{M}{2^l} 2^{i-1} < w(a,x) \leq \frac{M}{2^l} 2^{i}$ for $i\in \{0,1,\ldots,l\}$, according to Definition~\ref{def:sandwich}, for any $(x,y)\in S_a(w,l,x)$, we have $y_{i+1} = y_{i+2} = \ldots = y_l = 0$. This makes $|S_a(w,l,x)|=2^i$, which further implies that
\begin{equation}
  \frac{M}{2^l}\cdot \frac{|S_a(w,l,x)|}{2} < w(a,x) \leq \frac{M}{2^l} \cdot |S_a(w,l,x)|,
\end{equation}
or equivalently,
\begin{equation}
w(a,x) \leq \frac{M}{2^l} \cdot |S_a(w,l,x)| < 2w(a,x).
\end{equation}
Case (ii): If $w(a,x)\leq\frac{M}{2^{l+1}}$, we have $|S_a(w,l,x)|=1$. In other words,
\begin{equation}
  w(a,x) \leq 2w(a,x)\leq 2\frac{M}{2^{l+1}}|S_a(w,l,x)|= \frac{M}{2^{l}} |S_a(w,l,x)|.
\end{equation}
Also, $M 2^{-l} |S_a(w,l,x)| = M 2^{-l} \leq 2w(a,x)+M2^{-l}$. Hence, the following bound holds in both cases (i) and (ii): 
\begin{equation}
  w(a,x)\leq \frac{M}{2^l} |S_a(w,l,x)| \leq 2w(a,x)+M2^{-l}.\label{eq:weighted}
\end{equation}
The lemma holds by summing up over $\mathcal{X}$ and maximizing over $\mathcal{A}$ on all sides of Inequality~\ref{eq:weighted}.
\end{proof}
With the result of Lemma~\ref{lem:weighted}, we are ready to prove the following approximation result:
\begin{thm}
  Suppose there is an algorithm that gives a $c$-approximation to solve the unweighted problem: $\max_a \sum_{(x,y)}w'_l(a,x,y)$, then we have a $3c$-approximation algorithm to solve the weighted Marginal MAP problem $\max_a\sum_xw(a,x)$.
\end{thm}
\begin{proof}
  Let $l=n$ in Lemma~\ref{lem:weighted}. By definition $M = \max_{a,x}w(a,x) \leq \max_a\sum_xw(a,x)$, we have:
{\small
\begin{equation*}
\max_a\sum_xw(a,x)\leq \frac{M}{2^l}\max_a \sum_{(x,y)}w'_l(a,x,y) \leq 2\max_a\sum_xw(a,x)+M  \leq 3\max_a\sum_xw(a,x).
\end{equation*}
}
This is equivalent to:
{\small
\begin{equation*}
\frac{1}{3} \cdot \frac{M}{2^l}\max_a \sum_{(x,y)}w'_l(a,x,y) \leq \max_a\sum_xw(a,x)\leq \frac{M}{2^l}\max_a \sum_{(x,y)}w'_l(a,x,y).
\end{equation*}
}
\end{proof}

\vspace{-0.4in}
\section{Experiments}
\setlength{\textfloatsep}{6pt}

We evaluate our proposed algorithm $\mathtt{XOR\_MMAP}$ against two
baselines -- the Sample Average Approximation (SAA)
\cite{Sheldon10NetworkDesign} and the Mixed Loopy Belief Propagation
(Mixed LBP) \cite{LiuI13VariationalMMAP}. These two baselines are
selected to represent the two most widely used classes of methods that
approximate the embedded sum in MMAP problems in two different ways. SAA
approximates the intractable sum with a finite number of samples,
while the Mixed LBP uses a variational approximation.
We obtained the Mixed LBP implementation from the author of
\cite{LiuI13VariationalMMAP} and we use their default parameter
settings.
%
%
Since Marginal MAP problems are in general very hard and there is currently no exact solver that scales to reasonably large instances, our main comparison is on the
relative optimality gap: we first obtain the solution
$a_{method}$ for each approach. Then we compare the difference in
objective function $\log \sum_{x\in \setX} w(a_{method}, x) - \log \sum_{x\in \setX} w(a_{best}, x)$, in which $a_{best}$ is the best solution
among the three methods. Clearly a better algorithm will find a vector $a$
which yields a larger objective function.
The counting problem under a fixed solution $a$ is solved
using an exact counter ACE \cite{ChaviraDJ06Ace}, which is only used
for comparing the results of different MMAP solvers.

\begin{figure}[tb]
    \centering
    \includegraphics[width = 0.47\linewidth]{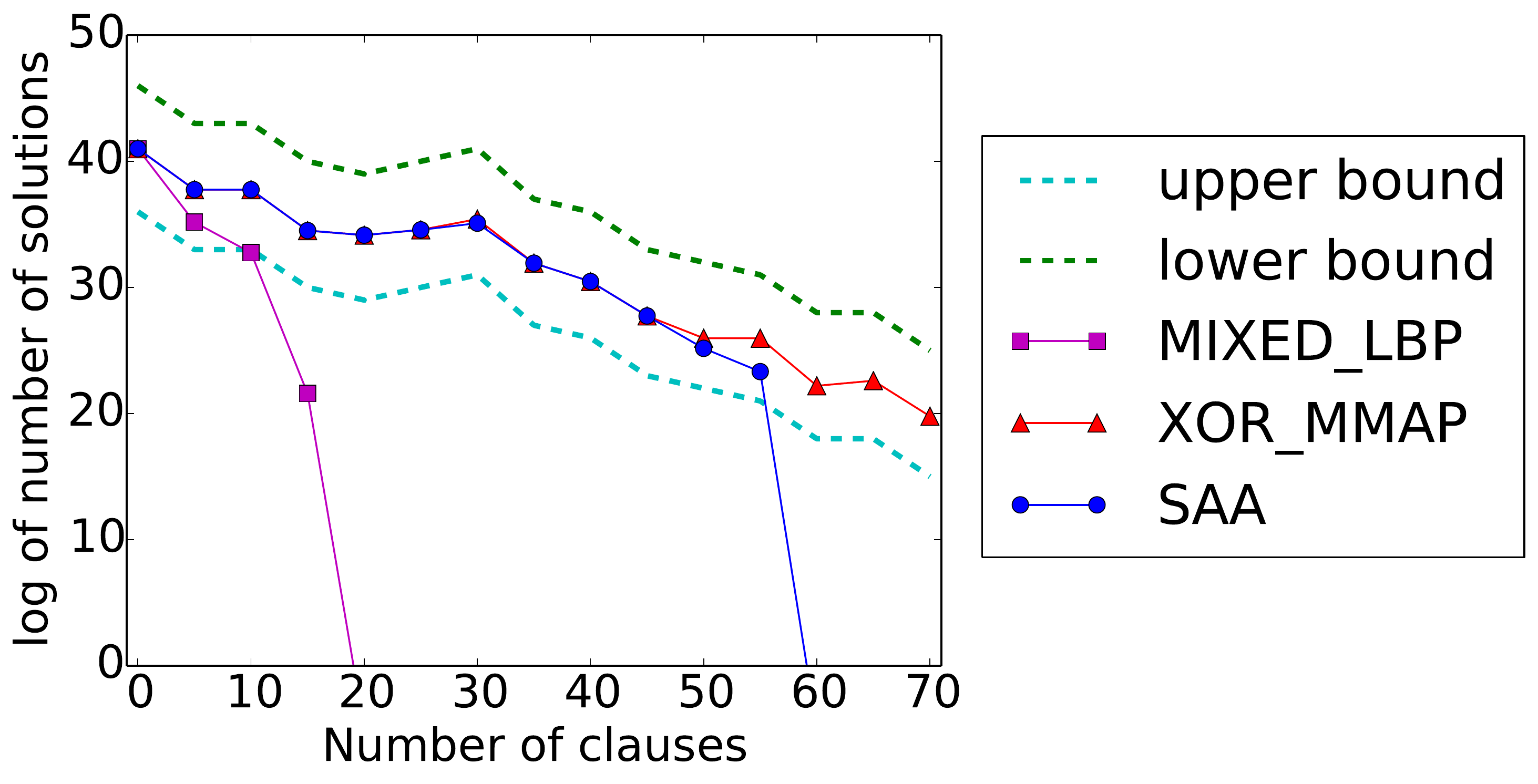}
    \includegraphics[width = 0.3\linewidth]{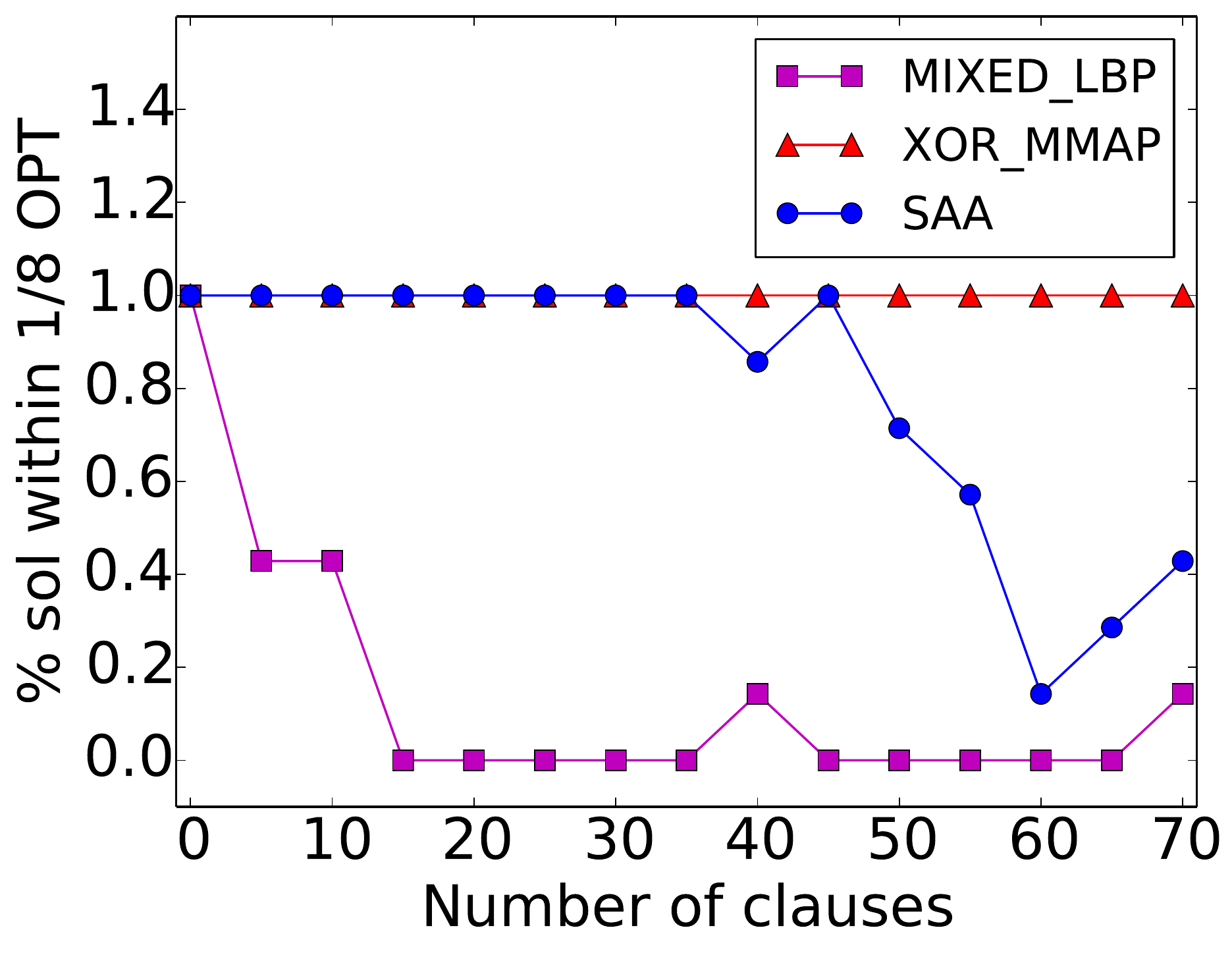}
    \caption{(Left) On median case, the solutions $a_0$ found by the proposed Algorithm $\mathtt{XOR\_MMAP}$ have higher objective $\sum_{x\in \setX} w(a_0, x)$ than the solutions found by SAA and Mixed LBP, on random 2-SAT instances with 60 variables and various number of clauses. Dashed lines represent the proved bounds from $\mathtt{XOR\_MMAP}$. (Right) The percentage of instances that each algorithm can find a solution that is at least 1/8 value of the best solutions among 3 algorithms, with different number of clauses.}
    \label{fig:exp_sat}
    \vspace{-0.2in}    
\end{figure}
\begin{figure}[tb]
    \centering
    \includegraphics[width = 0.32\linewidth]{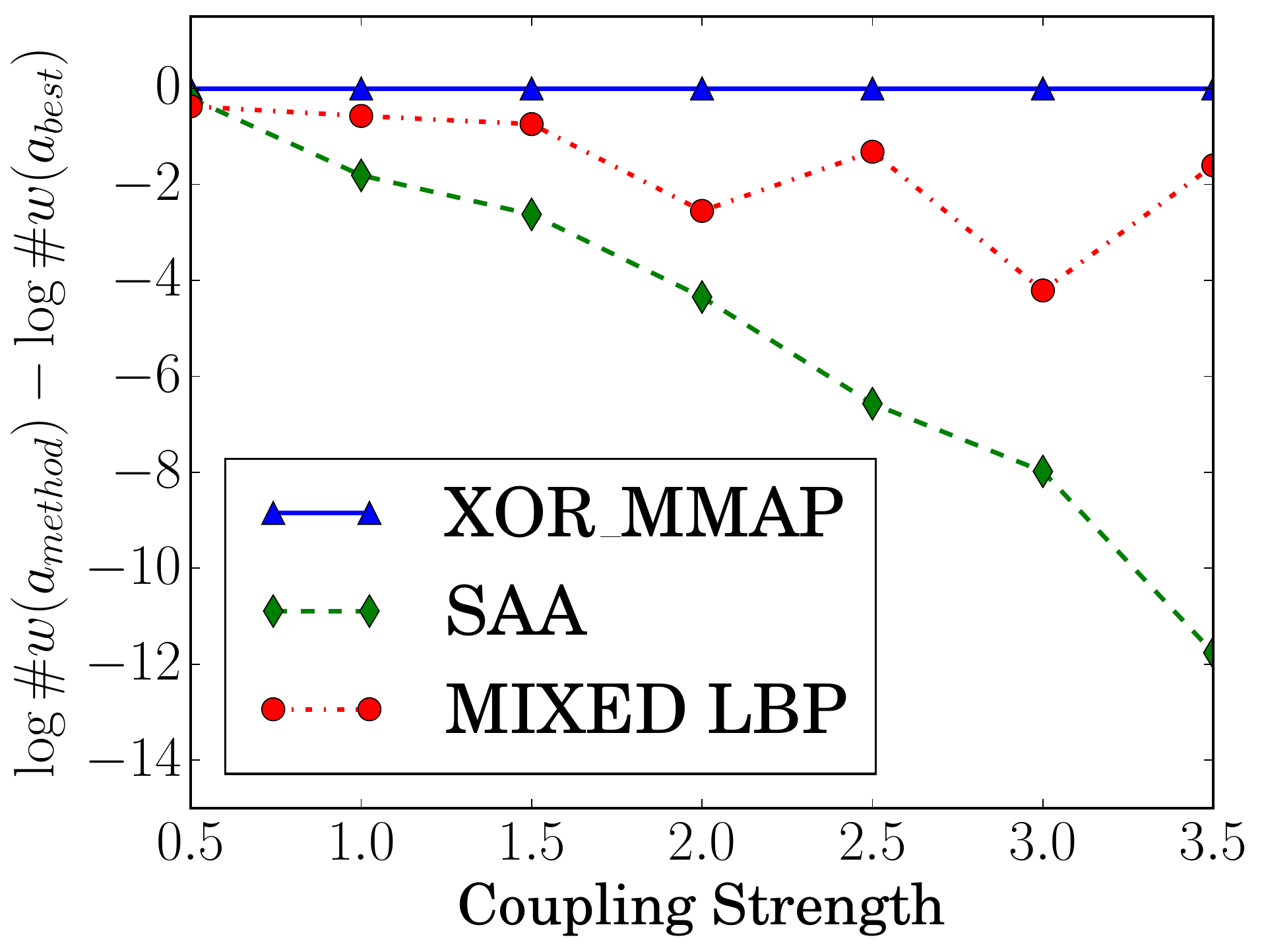}
    \includegraphics[width = 0.32\linewidth]{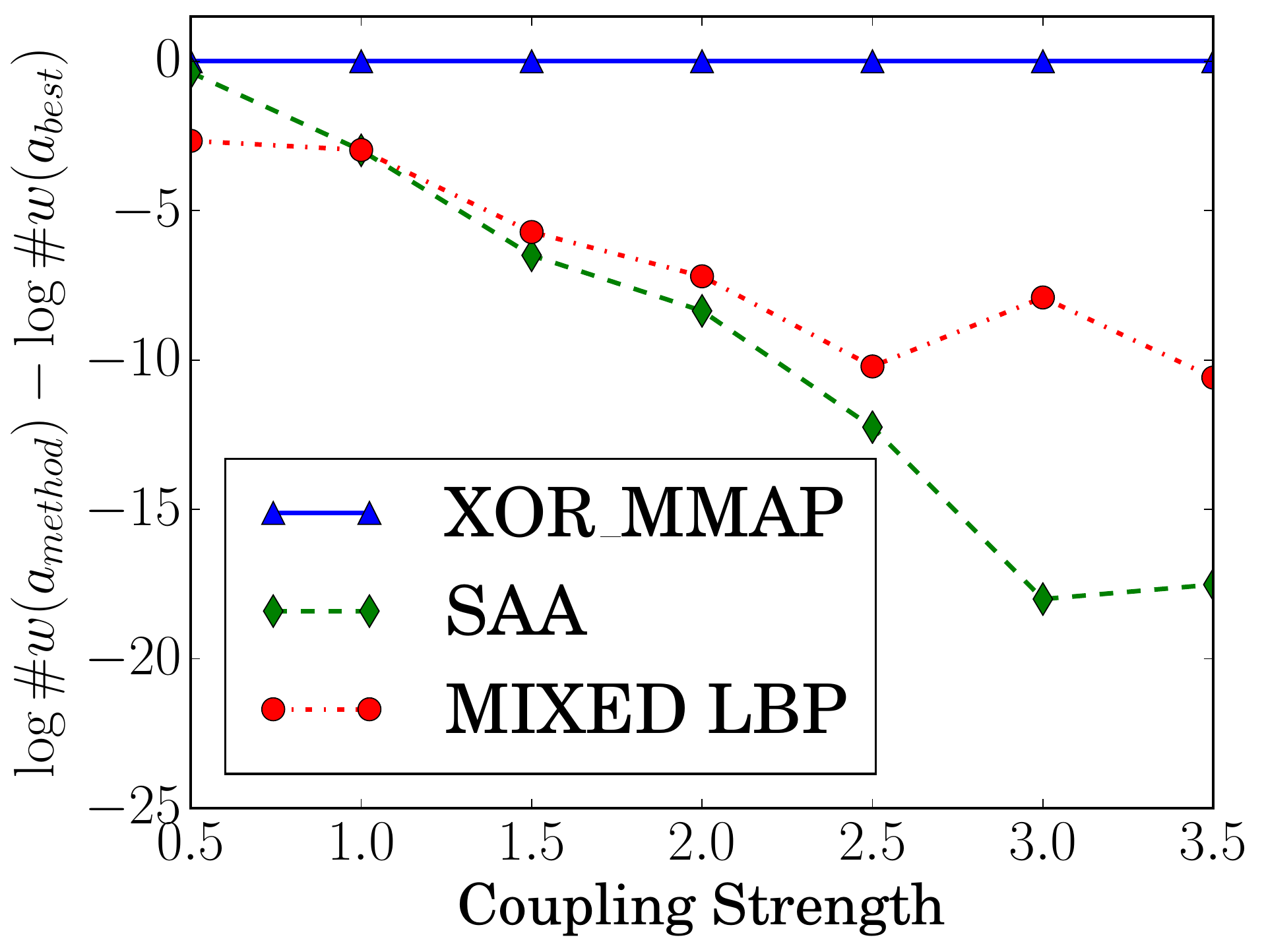}
    \includegraphics[width = 0.32\linewidth]{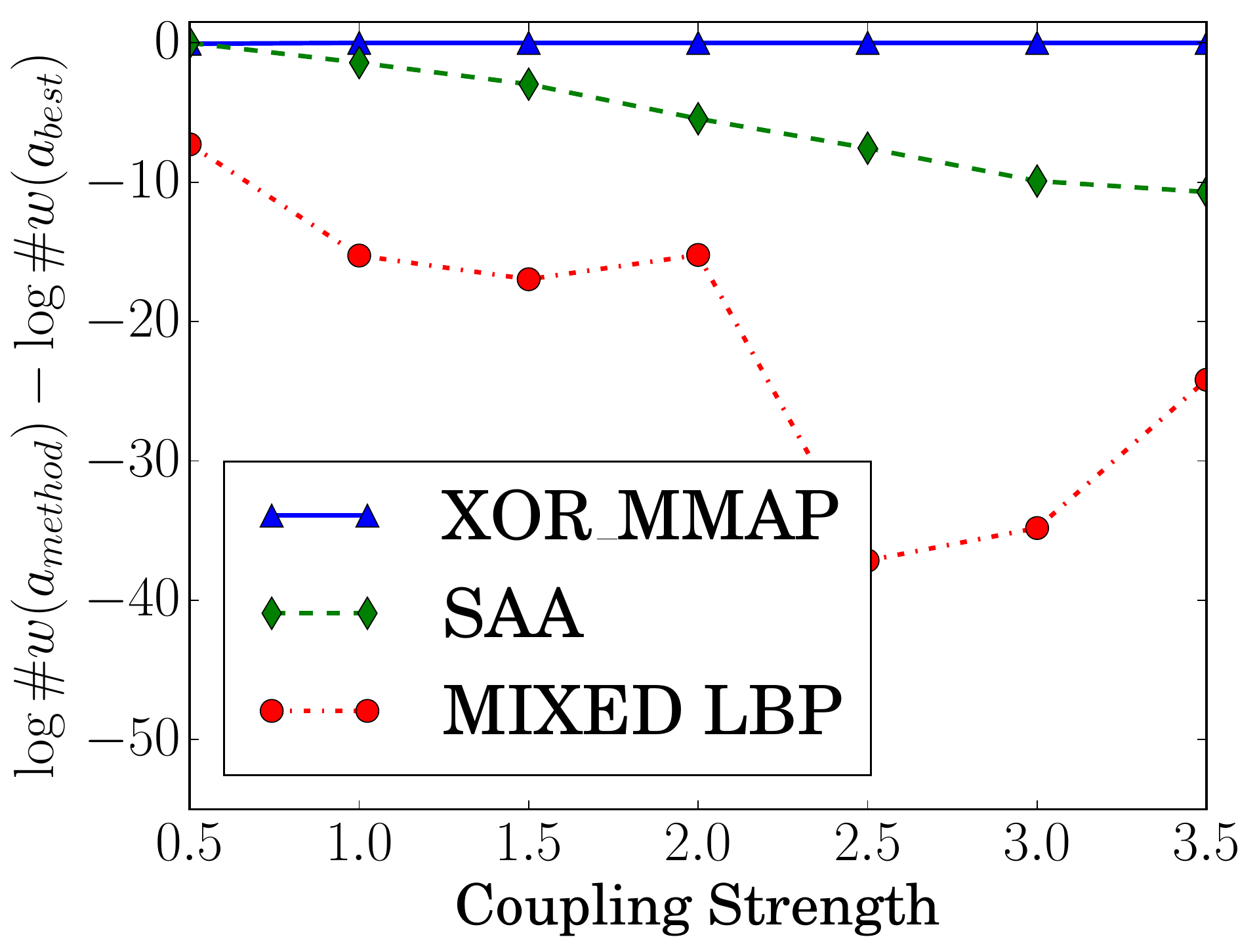}
    \includegraphics[width = 0.32\linewidth]{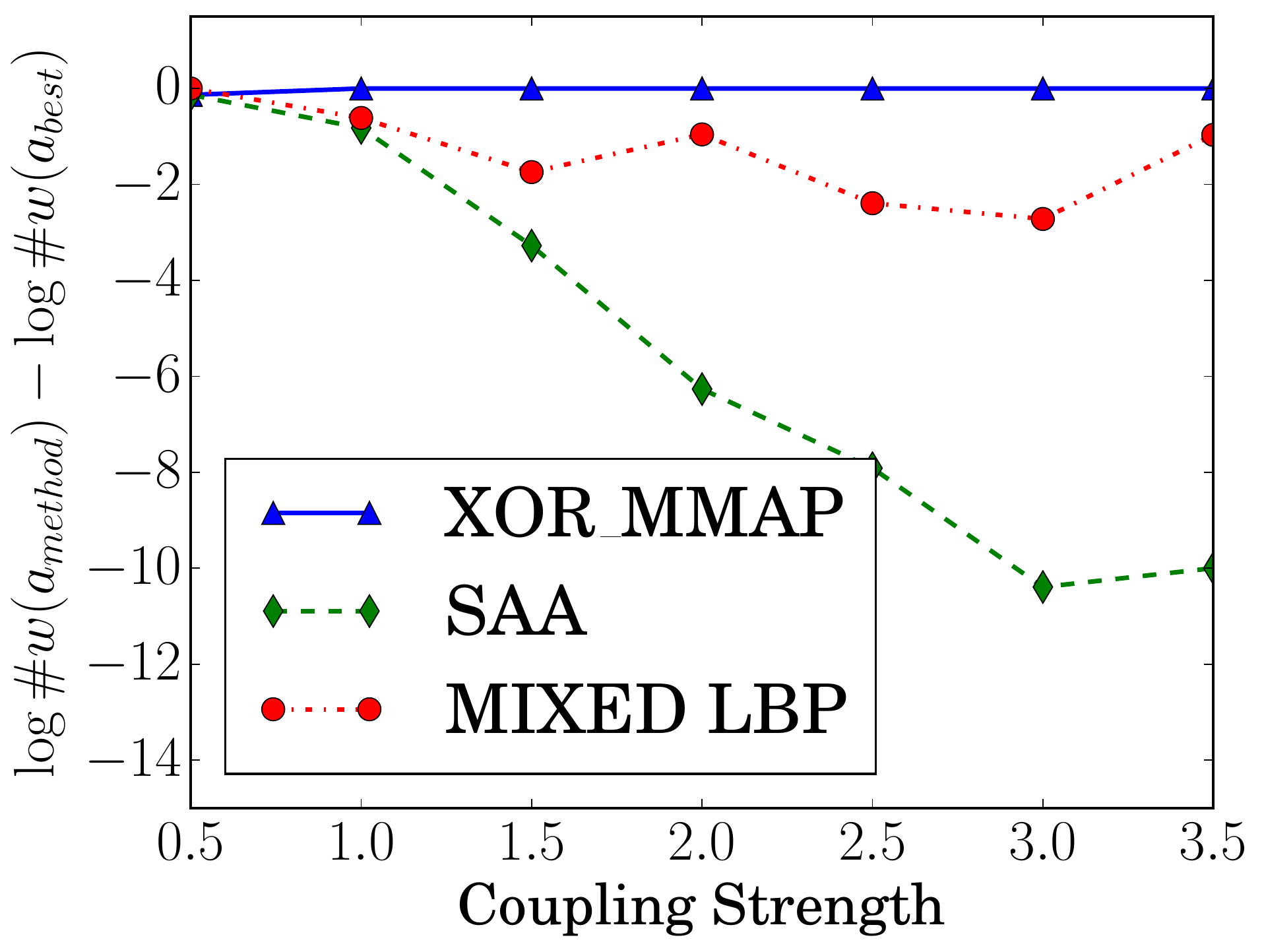}
    \includegraphics[width = 0.32\linewidth]{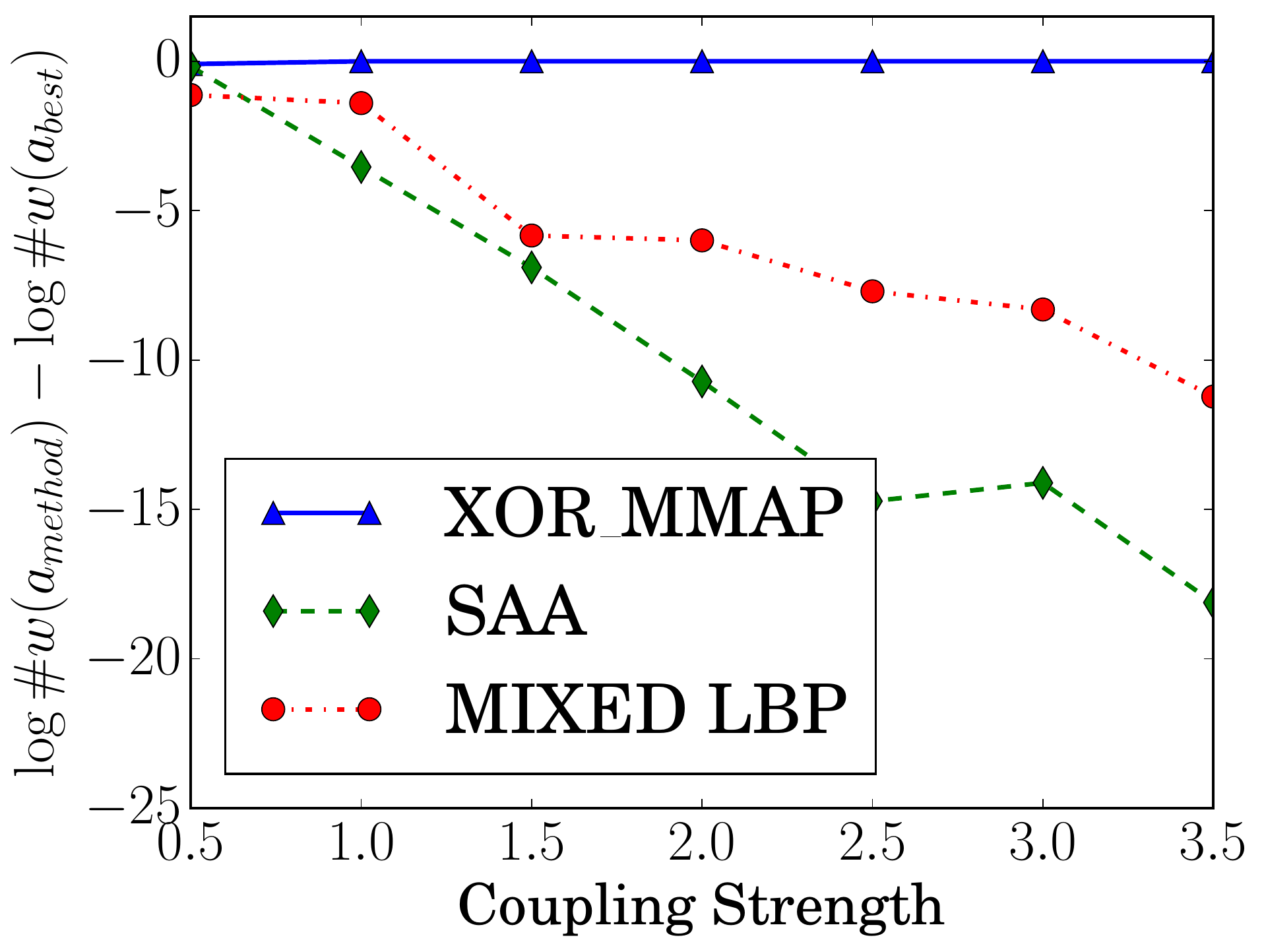}
    \includegraphics[width = 0.32\linewidth]{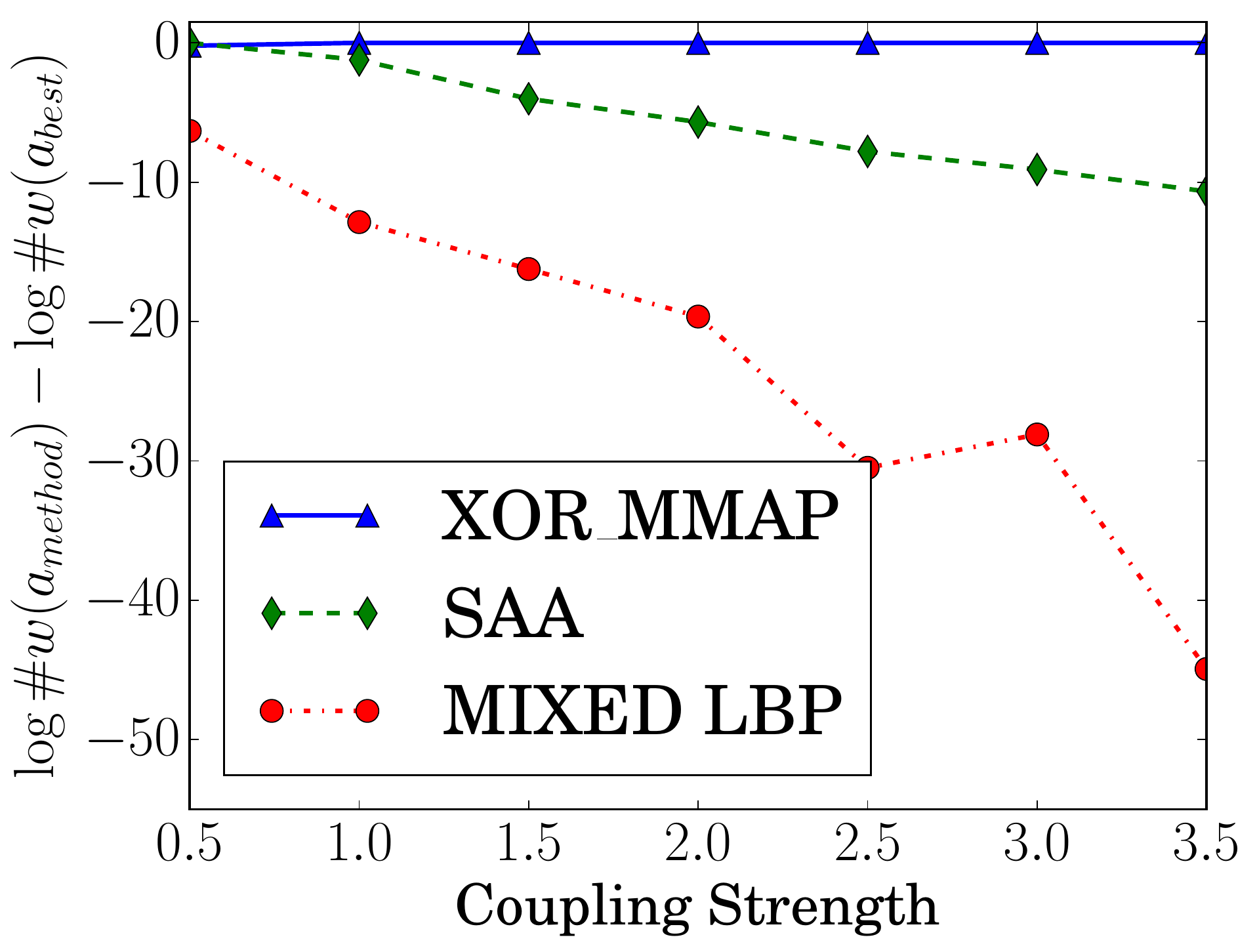}
    \caption{On median case, the solutions $a_0$ found by the proposed Algorithm $\mathtt{XOR\_MMAP}$ are better than the solutions found by SAA and Mixed LBP, on weighted 12-by-12 Ising models with mixed coupling strength. (Up) Field strength 0.01. (Down) Field strength 0.1. (Left) 20\% variables are randomly selected for maximization. (Mid) 50\% for maximization. (Right) 80\% for maximization.}
    \label{fig:exp_ising}
\end{figure}

Our first experiment is on unweighted random 2-SAT instances. Here,
$w(a,x)$ is an indicator variable on whether the 2-SAT instance is
satisfiable. The SAT instances have 60 variables, 20 of which are
randomly selected to form set $\setA$, and the remaining ones form set
$\setX$. The number of clauses varies from 1 to 70. For a fixed number
of clauses, we randomly generate 20 instances, and the left panel of
Figure~\ref{fig:exp_sat} shows the median objective function
$\sum_{x\in \setX} w(a_{method}, x)$ of the solutions found by the three 
approaches.
We tune the constants of our $\mathtt{XOR\_MMAP}$ so it
gives a $2^{10} = 1024$-approximation ($2^{-5}\cdot sol \leq OPT \leq
2^{5}\cdot sol$, $\delta = 10^{-3}$). The upper and lower bounds are shown in dashed lines. SAA uses 10,000 samples.
%
On average, the running time of our algorithm is
reasonable. When enforcing the $1024$-approximation bound,
the median time for a single $\mathtt{XOR\_k}$ procedure is in seconds,
although we occasionally have long runs (no more than 30-minute timeout).
%

As we can see from the left panel of Figure~\ref{fig:exp_sat}, both
Mixed LBP and SAA match the performance of our proposed
$\mathtt{XOR\_MMAP}$ on easy instances. However, as the number of
clauses increases, their performance quickly deteriorates. In fact, for
instances with more than 20 (60) clauses, typically the $a$
vectors returned by Mixed LBP (SAA) do not yield non-zero solution
values. Therefore we are not able to plot their performance
beyond the two values. At the same time, our algorithm
$\mathtt{XOR\_MMAP}$ can still find a vector $a$ yielding over
$2^{20}$ solutions on larger instances with more than 60 clauses,
while providing a 1024-approximation.

Next, we look at the performance of the three algorithms on weighted
instances. Here, we set the number of replicates $T=3$ for our
algorithm $\mathtt{XOR\_MMAP}$, and we repeatedly start the algorithm
with an increasing number of XOR constraints $k$, until it completes
for all $k$ or times out in an hour.  For SAA, we use 1,000 samples,
which is the largest we can use within the memory limit. All algorithms are
given a one-hour time and a 4G memory limit.

The solutions found by $\mathtt{XOR\_MMAP}$ are considerably better
than the ones found by Mixed LBP and SAA on weighted instances.
Figure~\ref{fig:exp_ising} shows the performance of the three algorithms
on 12-by-12 Ising models with mixed coupling strength, different
field strengths and number of variables to form set $\setA$.
All values in the figure are median values across 20 instances
(in $\log_{10}$).
In all 6 cases in Figure~\ref{fig:exp_ising}, our algorithm
$\mathtt{XOR\_MMAP}$ is the best among the three approximate
algorithms.
In general, the difference in performance increases as 
the coupling strength increases.
These instances are challenging for the state-of-the-art complete
solvers. For example, the state-of-the-art exact solver AOBB with
mini-bucket heuristics and moment matching \cite{Marinescu2015AOBB}
runs out of 4G memory on 60\% of instances with 20\%
variables randomly selected as max variables.
%
%
%
We also notice that the solution found by our $\mathtt{XOR\_MMAP}$ is
already close to the ground-truth. On smaller 10-by-10 Ising models
which the exact AOBB solver can complete within the memory limit, the
median difference between the log10 count of the solutions found by $\mathtt{XOR\_MMAP}$ and those found by the exact solver is 0.3, while the differences between the solution values of $\mathtt{XOR\_MMAP}$ against those of the Mixed BP or SAA are on the order of 10.

\begin{figure}[tb]
    \centering
    \begin{subfigure}[b]{0.33\linewidth}
    \raisebox{0.4\height}{\includegraphics[width = 1\linewidth]{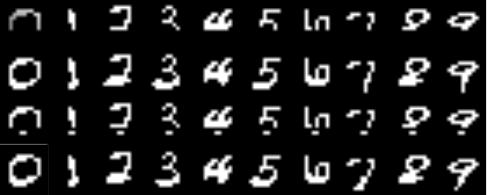}}
    \end{subfigure}
    \begin{subfigure}[b]{0.3\linewidth}
    \raisebox{0.1\height}{\includegraphics[width = 1\linewidth]{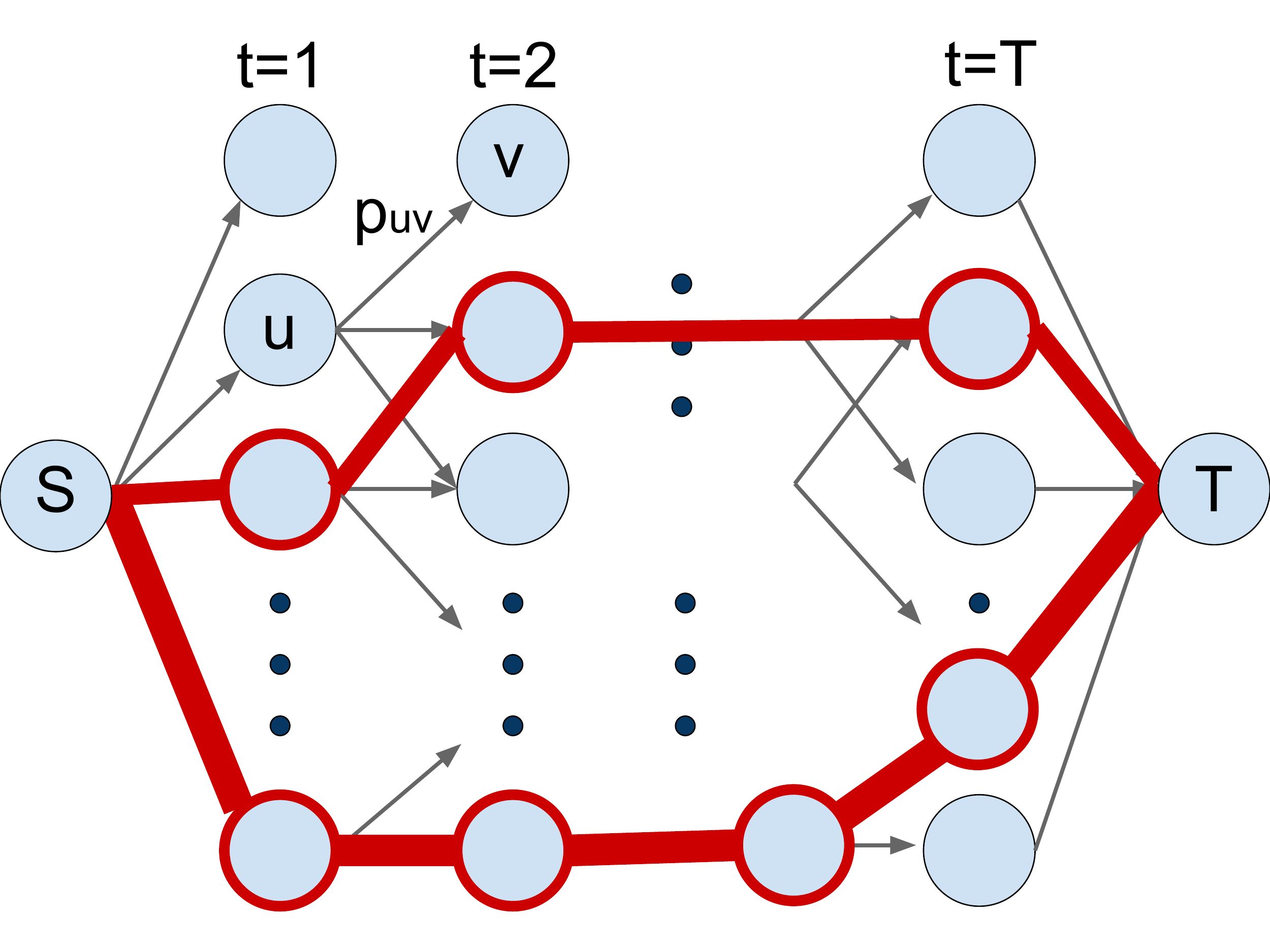}}
    \end{subfigure}
    \begin{subfigure}[b]{0.32\linewidth}
    \includegraphics[width = 1\linewidth]{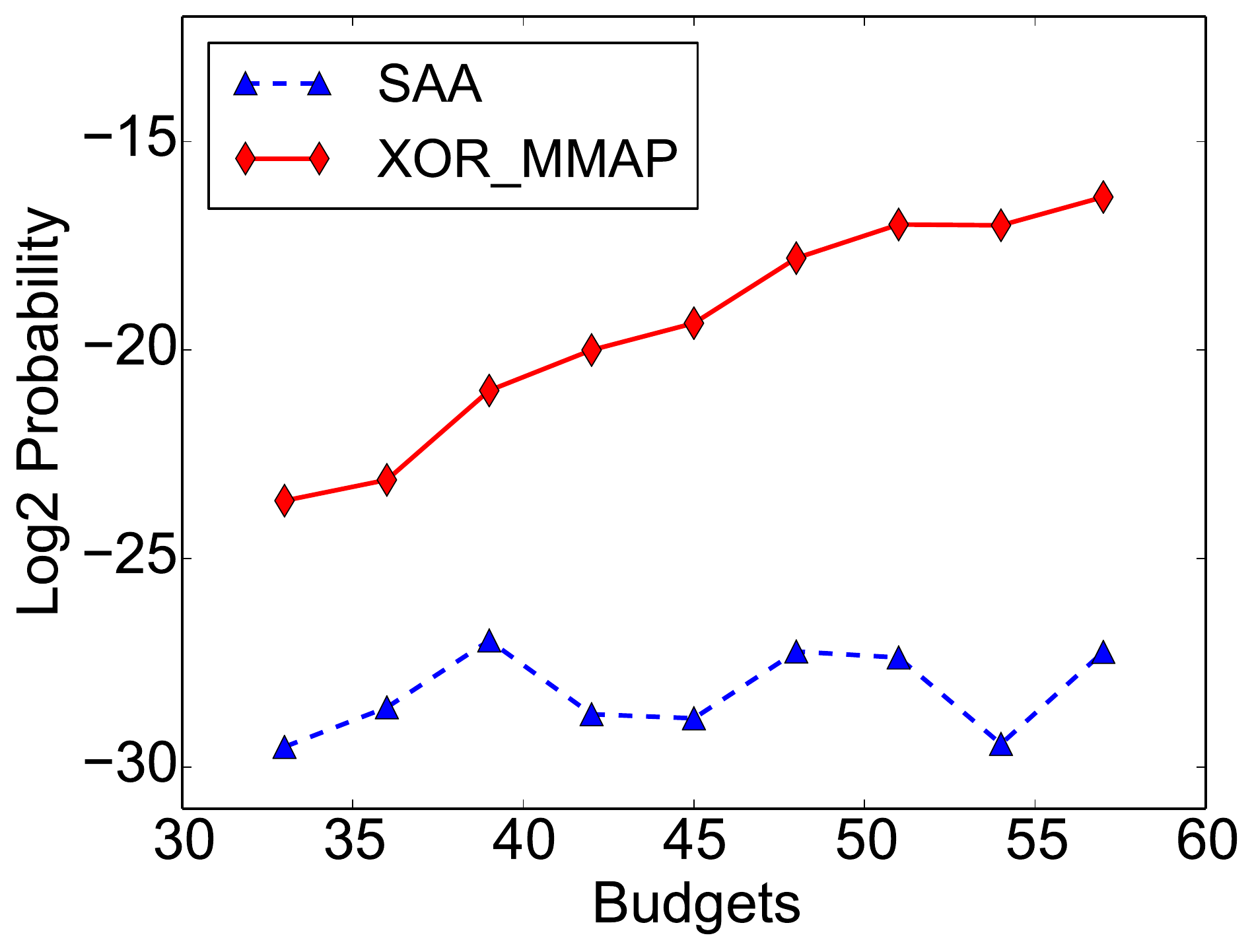}
    \end{subfigure}
    \vspace{-0.15in}    
    \caption{(Left) The image completion task. Solvers are given
      digits of the upper part as shown in the first row. Solvers need to complete the digits based on a two-layer deep belief network and the upper part. (2nd Row) completion given by $\mathtt{XOR\_MMAP}$. (3rd Row) SAA. (4th Row)  Mixed Loopy Belief Propagation. (Middle) Graphical illustration of the network cascade problem. Red circles are nodes to purchase. Lines represent cascade probabilities. See main text. (Right) Our $\mathtt{XOR\_MMAP}$ performs better than SAA on a set of network cascade benchmarks, with different budgets.}
    \label{fig:exp_dbn}
\end{figure}

We also apply the Marginal MAP solver to an image completion task. We
first learn a two-layer deep belief network \cite{BengioLPL06LayerTraining,Hinton06RBM} from a 14-by-14
MNIST dataset. Then for a binary image that only contains the upper
part of a digit, we ask the solver to complete the lower part, based
on the learned model.
This is a Marginal MAP task, since one needs to integrate over the
states of the hidden variables, and query the most likely states of
the lower part of the image. Figure~\ref{fig:exp_dbn} shows the result
of a few digits. As we can see, SAA performs poorly. In most cases, it
only manages to come up with a light dot for all 10 different
digits. Mixed Loopy Belief Propagation and our proposed
$\mathtt{XOR\_MMAP}$ perform well. The good performance of Mixed LBP
may be due to the fact that the  weights on pairwise factors in the learned deep belief network are not very combinatorial. 
%

Finally, we consider an application that applies decision-making
into machine learning models. 
This network design application maximizes the spread of cascades in
networks, which is important in the domain of social networks and
computational sustainability.
In this application, we are given a stochastic graph, in which the
source node at time $t=0$ is affected. For a node $v$ at time $t$, it
will be affected if one of its ancestor nodes at time $t-1$ is
affected, and the configuration of the edge connecting the two nodes
is ``on''. An edge connecting node $u$ and $v$ has probability
$p_{u,v}$ to be turned on. A node will not be affected if it is not
purchased.
Our goal is to purchase a set of nodes within a finite
budget, so as to maximize the probability that the target node is
affected.
We refer the reader to \cite{Sheldon10NetworkDesign} for more
background knowledge.
This application cannot be captured by graphical models due to global
constraints. Therefore, we are not able to run mixed LBP on this
problem. We consider a set of synthetic networks, and compare the
performance of SAA and our $\mathtt{XOR\_MMAP}$ with different budgets.
As we can see from the right panel of Figure~\ref{fig:exp_dbn}, the
nodes that our $\mathtt{XOR\_MMAP}$ decides to purchase result in higher
probabilities of the target node being affected, compared to SAA. 
Each dot in the figure is the median value over 30 networks generated
in a similar way.
%

\vspace{-0.1in}
\section{Conclusion}
\vspace{-0.15in}
We propose $\mathtt{XOR\_MMAP}$, a novel constant approximation
algorithm to solve the Marginal MAP problem. Our approach represents
the intractable counting subproblem with queries to NP oracles,
subject to additional parity constraints. In our algorithm, the entire
problem can be solved by a single optimization.
We evaluate our approach on several machine learning and decision-making applications. We are able to show that $\mathtt{XOR\_MMAP}$
outperforms several state-of-the-art Marginal MAP solvers.
%
$\mathtt{XOR\_MMAP}$ provides a new angle to solving the Marginal MAP problem, opening the door to new research directions and applications in real world domains. 


%
\vspace{-0.15in}
\subsubsection*{Acknowledgments}
\vspace{-0.1in}
This research was supported by National Science Foundation (Awards \#0832782, 1522054, 1059284, 1649208) and Future of Life Institute (Grant 2015-143902). 

{
\small
\bibliography{BayesianOpt2,BooleanFunction,message_passing_parity_constraint}
\bibliographystyle{plain}
}

\newpage
\section*{Appendix}

\subsection*{Proof of Claim 2 of Lemma 3.3}
The proof is almost the same as Claim 1, except that we need to use a union bound to let the property hold for all $a \in \mathcal{A}$ simultaneously. As a result, the success probability will be $1-\frac{\delta}{n}$ instead of $1-\frac{\delta}{n 2^m}$. 
Let $Y_i(a) = \max_{x^{(i)}\in \mathcal{X}, h^{(i)}_{k+c}(x^{(i)})=\mathbf{0}} w(a_0,x^{(i)})$ for $i=1,\ldots,T$. If for all $a\in \mathcal{A},\ \#w(a)< 2^{k}$, then for any fixed $a_0$,
          \begin{equation}
            \pr{\sum_{i=1}^{T}Y_i(a_0) > \frac{T}{2}}\leq e^{-D\left(\frac{1}{2}\|\frac{2^c}{(2^c-1)^2}\right)T}\leq \frac{\delta}{n2^m}.
          \end{equation}
    Using Union bound, we have:
    \begin{equation}
        \pr{\max\limits_{a\in \mathcal{A}} \sum_{i=1}^{T}Y_i(a)>\frac{T}{2}}\leq \pr{\exists a\in \mathcal{A}, \sum_{i=1}^{T}Y_i(a)>\frac{T}{2}}\leq 2^m \frac{\delta}{n2^m}=\frac{\delta}{n}. 
    \end{equation}
    Therefore, w.p.  $1-\frac{\delta}{n}$, we have $\forall a$, $\sum_{i=1}^{T}Y_i(a) \leq \frac{T}{2}$, which implies that $\mathtt{XOR\_K}$ returns $\mathbf{false}$.

\subsection*{Algorithm Variants to Reduce the Number of Replicates $T$}

\begin{table}[h]
\begin{minipage}[t]{.5\textwidth}
\begin{algorithm}[H]
Sample $T$ pair-wise independent hash functions \\
~~~~~~~~~~$h_k^{(1)},h_k^{(2)},\ldots,h_k^{(T)}:\mathcal{X}\to \{0,1\}^{k}$\;
Query Oracle
\begin{equation*}
\begin{split}
   \max\limits_{a\in\mathcal{A},x^{(i)}\in \mathcal{X}}& \sum_{i=1}^{T}w(a,x^{(i)})\\
  \textbf{s.t.}& \quad h^{(i)}_k(x^{(i)})=\mathbf{0}, \quad i=1,\ldots,T
\end{split}
\end{equation*}
Return \textbf{true} if the max value is no less than $q$, otherwise return \textbf{false}.
\caption{$\mathtt{XOR\_K+}$($w:\mathcal{A}\times \mathcal{X}\to \{0,1\},k,T, q$)}\label{alg:COMP+}
\end{algorithm}
\end{minipage}
~~~~~
\begin{minipage}[t]{.45\textwidth}
\begin{algorithm}[H]
$k = n$\;
\While{$k>0$}{
     Run $\mathtt{XOR\_K}$($w,k,T$) $r$ times. Each time $\mathtt{XOR\_K}$ returns either \textbf{true} or \textbf{false}\;
  \If{half independent trials return \textbf{true}}{
    Return $2^k$\;
  }
  $k\gets k-1$\;
}
Return 1.
\caption{$\mathtt{XOR\_MMAP+}$($w:\mathcal{A}\times \mathcal{X}\to \{0,1\}$,$n=\log_2|\mathcal{X}|$,$m=\log_2|\mathcal{A}|$,$T$,$r$)}\label{alg:XOR_MMAP+}
\end{algorithm}
\end{minipage}
\end{table}

We develop a few variants of the original algorithm so as to reduce $T$ -- the number of replicates to guarantee a constant approximation bound. 

\noindent {\bf Option 1: Use Binary Search in $\mathtt{XOR\_MMAP}$}: 
Using binary search, we only need to check $\log_2 n$ outcomes (instead of $n$ outcomes) of calls to function $\mathtt{XOR\_K}$($w,k,T$) to nail down the correct $k$. Because our proof to Theorem 3.2 is relies on a probabilistic bound that all outcomes of the calls to $\mathtt{XOR\_K}$($w,k,T$) satisfy the two claims of Lemma 3.3, with fewer checks to the outcomes of $\mathtt{XOR\_K}$($w,k,T$), we can reduce the number of replicates $T$ to $ \lceil\frac{m\ln2 +\ln\log_2 n+\ln(1/\delta)}{\alpha^*(c)}\rceil$, while preserving the same probabilistic bound. 

\noindent {\bf Option 2: Run $\mathtt{XOR\_K}$ Multiple Times}:
Algorithm $\mathtt{XOR\_MMAP+}$ achieves similar approximation bounds with smaller $T$, but at the price of running the  $\mathtt{XOR\_K}$ procedure $r > 1$ times. 

\begin{thm}\label{thm:I_m_a_p}
 Let $p=\frac{2^c}{(2^c -1)^2}$. For $T\geq \frac{m\ln2 +\ln \frac{1}{p}}{\alpha^*(c)}$, $r \geq \frac{\ln\frac{n}{\delta}}{D(\frac{1}{2}\|p)}$, with probability $1-\delta$, $\mathtt{XOR\_MMAP+}$($w$,$n$,$m$,$T,r$) outputs a $2^{2c}$-approximation to the Marginal MAP Problem: $\max\limits_{a\in \mathcal{A}}\#w(a)$.
\end{thm}

\proofsketch
  By the same argument in the proof of \ref{lem:COMP}, for $T\geq \frac{m\ln2 +\ln\frac{1}{p}}{\alpha^*(c)}$, we have 
    \begin{itemize}
    \item Suppose $\exists a^* \in \mathcal{A}$, s.t. $\#w(a^*)\geq 2^{k}$, then with probability $1-\frac{p}{2^m}$, $\mathtt{XOR\_K}(w,k-c,T)$ returns $\mathbf{true}$.
    \item Suppose $\forall a_0\in \mathcal{A}$, $\#w(a_0) < 2^{k}$, then with probability  $1-p$, $\mathtt{XOR\_K}(w,k+c,T)$ returns $\mathbf{false}$.
  \end{itemize}
  
  Then we can use $r$ independent trials to amplify the success probability from $1-p$ to $1-e^{-D(\frac{1}{2}\|p)\cdot r}$, which is larger than $1-\delta/n$. Then we apply union bound for all $a\in \setA$, to guarantee a $2^{2c}$-approximation with probability larger than $1-\delta$.

\noindent {\bf Option 3: Biased $\mathtt{XOR\_K}$ Procedure}:
Furthermore, instead of returning $\mathbf{true}$ if at least $\lceil T/2 \rceil$ replicates are 1 in the procedure $\mathtt{XOR\_K}$, we develop $\mathtt{XOR\_K+}$ which returns $\mathbf{true}$ if no less than $q$ replicates are 1. Then $q$ becomes a parameter that we can tune. As a result, we can further reduce the number of trials $T$. 


\begin{lem}\label{lem:balance}
  Procedure $\mathtt{XOR\_K+}$($w,k,T,q$) satisfies the following properties:
    \begin{itemize}
    \item Suppose $\exists a^* \in \mathcal{A}$, s.t. $\#w(a^*)\geq 2^{k}$, then with probability $1-e^{-D(\frac{T-q+1}{T}\|\frac{2^c}{(2^c-1)^2})T}$, $\mathtt{XOR\_K+}(w,k-c,T,q)$ returns $\mathbf{true}$.
    \item Suppose $\forall a_0\in \mathcal{A}$, $\#w(a_0) < 2^{k}$, then with probability  $1-2^me^{-D(\frac{q}{T}\|\frac{2^c}{(2^c-1)^2})T}$, $\mathtt{XOR\_K+}(w,k+c,T,q)$ returns $\mathbf{false}$.
  \end{itemize}
\end{lem}

\proofsketch
  Let $X^{(i)}(a)$ and $Y^{(i)}(a)$ be defined in the same way as in the proof of Lemma \ref{lem:COMP}. Similarly we have 
  \begin{equation}\label{eq:++X}
    \pr{\max\limits_{a\in\mathcal{A}}\sum_{i=1}^{T}X^{(i)}(a) \leq q-1}\leq \pr{\sum_{i=1}^{T}X^{(i)}(a^*) \leq q-1} \leq e^{-D(\frac{T-q+1}{T}\|\frac{2^c}{(2^c-1)^2})T},
  \end{equation}
  
  and,
    \begin{equation}\label{eq:++Y}
    \pr{\max\limits_{a\in\mathcal{A}}\sum_{i=1}^{T}Y^{(i)}(a) \geq q} \leq \pr{\forall a \in \mathcal{A}, \sum_{i=1}^{T}Y^{(i)}(a) \geq q} \leq 2^me^{-D(\frac{q}{T}\|\frac{2^c}{(2^c-1)^2})T}.
  \end{equation}
  

Based on Lemma \ref{lem:balance}, we can reduce the number of independent trials $T$, by picking a $q$ larger than $\frac{T}{2}$, which balances the right hand sides of \eqref{eq:++X} and \eqref{eq:++Y}. In other words, we pick $q^*(T)$ to minimize the max of the two values:
\begin{equation}
  q^*(T)=\argmin\limits_{q > \frac{T}{2}}\max\{e^{-D(\frac{T-q+1}{T}\|\frac{2^c}{(2^c-1)^2})T},2^me^{-D(\frac{q}{T}\|\frac{2^c}{(2^c-1)^2})T}\}.
\end{equation} 

The $q^*(T)$ picked in this way can further reduce the number of independent trials.

\subsection*{Represent  $\max_{a, x^{(i)}} \sum_{i} w(a, x^{(i)})$ When the Unweighted $w(a, x^{(i)})$ Are Specified by CNF}

The idea is to add extra binary variables for each replicate. Suppose $w(a, x^{(i)})$ is 1 if and only if the following CNF:
\begin{equation*}
  cl_{i,1} \wedge cl_{i,2} \ldots \wedge cl_{i,p}
\end{equation*}
is satisfiable. In this CNF, $cl_{i,j}$ ($j=1,\ldots,p$) are clauses that range over variables $x^{(i)}$ and $a$. Introduce extra binary variable $y_i$ for replicate $w(x^{(i)}, a)$. Then we can augment clause $cl_{i,j}$ to $$cl_{i,j} \vee (y_i=0),$$ 
and replace the global objective function from $\max_{a,x^{(i)}} \sum_{i} w(a, x^{(i)})$ to: 
$$\max_{a,x^{(i)}} \sum_{i} y_i ~~~\mbox{ subject to }~~~ (cl_{i,1} \vee (y_i =0)) \wedge  \ldots \wedge (cl_{i,p} \vee (y_i=0))~~~~ \forall i.$$
Similar encoding exists when we are solving $\max_{a,x^{(i)}} \sum_{i} w(a, x^{(i)})$ subject to extra parity constraints.

\end{document}